\documentclass[onefignum,onetabnum]{siamonline250211}

\usepackage{amsfonts}
\usepackage{amssymb}
\usepackage{mathtools}
\usepackage{booktabs}
\usepackage{graphicx}
\usepackage{microtype}
\usepackage{enumitem}

\ifpdf
  \DeclareGraphicsExtensions{.pdf,.png,.jpg}
\fi

\newsiamthm{assumption}{Assumption}
\newsiamremark{remark}{Remark}

\newcommand{\I}{\mathcal I}
\newcommand{\Y}{\mathcal Y}
\newcommand{\W}{\mathcal W}
\newcommand{\R}{\mathcal R}

\newcommand{\Pp}{\mathbb P}

\headers{Common-Witness Image Auditing}{U. Faghihi and A. Saki}

\title{Common-Witness Certificates and Sharp Feature Bounds for
Counterfactual Image Auditing}
\author{Usef Faghihi\thanks{Department of Mathematics and Computer Science,
Universit\'e du Qu\'ebec \`a Trois-Rivi\`eres, Trois-Rivi\`eres, Quebec,
Canada (\email{usef.faghihi@uqtr.ca}). Both authors contributed equally:
each contributed 50\% of the theoretical work and 50\% of the experimental
work.}
\and Amir Saki\footnotemark[1]}

\ifpdf
\hypersetup{
  pdftitle={Common-Witness Certificates and Sharp Feature Bounds for Counterfactual Image Auditing},
  pdfauthor={Usef Faghihi and Amir Saki}
}
\fi

\begin{document}
\maketitle

\begin{abstract}
Structured image editors may satisfy every regional plausibility constraint
separately although no single latent explanation is compatible with the
complete output. We formulate this local-to-global failure through a
common-witness grade and its witness nerve. The framework separates auditing
from causal identification: shared exogeneity alone permits every coupling of
the regime marginals, whereas a scientifically justified witness relation
yields exactly the relation-supported couplings and hence sharp feature-level
partial-identification bounds. For quasiconvex regional losses, classical
Helly theory gives finite incompatibility certificates. For the labeled
witness structures generated by the audit, we derive a heterogeneous
action-stratified certificate, a sharp tolerant certificate for finite
witness atlases, and a blocker-hypergraph repair formula. Fractional Helly
theory converts dense local compatibility into a large jointly coherent
subset. Simultaneous confidence regions for regime marginals provide
finite-sample outer confidence bounds for the complete identified interval;
Bonferroni--Clopper--Pearson bands give a nonasymptotic multinomial
construction. In controlled MNIST rotations and finite paired studies on
Morpho-MNIST and smallNORB, archived summaries report coherent-pair acceptance
of 0.9593--0.9778, while regional-action patchworks retain local acceptance
of 0.9702--0.9804 but have global acceptance of 0--0.1138. Synthetic studies
exercise sharp bounds, certificate recovery, and structured computation up to
$10^4$ feature states and $10^5$ regional constraints. The method audits a
prespecified feature relation; it does not identify unrestricted pixel-level
counterfactuals.
\end{abstract}



\section{Introduction}
\label{sec:introduction}

Counterfactual image editing asks what the image of the same unit would have
been under another intervention. In an augmented structural causal model
(ASCM), both potential images are evaluated at the same exogenous state.
That semantic convention does not reveal their joint distribution. Pan and
Bareinboim show that unrestricted pixel-level counterfactuals are generally
not identified from image--label samples, even when a latent causal diagram
is supplied \cite{pan2024counterfactual}. Their feature-based relaxation is
therefore a bound on declared care-set quantities, not recovery of the true
pixel coupling.

We study a complementary auditing failure. An edited image can satisfy each
protected-region constraint with a different latent explanation while no one
explanation satisfies all regions simultaneously. Formally, a witness for
each region need not be one witness for every region. This
\emph{global-witness gap} is invisible to an audit that aggregates independent
local passes.

The gap becomes scientifically informative only when the witness family is
anchored outside the candidate pair. Examples include a validated renderer,
paired interventions, or a design guarantee. A similarity score fitted only
to unpaired images is not, by itself, cross-world information. We encode an
externally justified witness family by regional costs, use their common
sublevel intersections to audit an image pair, and project the resulting
admissibility rule to a prespecified finite feature. The projection is then
an explicit assumption in a support-only partial-identification model; it is
not inferred from the causal graph.

The contribution is a certificate-to-inference pipeline:

\begin{enumerate}[leftmargin=1.6em]
  \item We turn regional costs into a graded feasibility complex. For the
  labeled witness structures produced by this audit, we derive exact
  certificate-size and repair formulas: a sharp $q(s+1)$ certificate for a
  finite atlas with $q$ witnesses and $s$ exceptional roles, a
  blocker-hypergraph description of minimal explanations, and a heterogeneous
  action-stratified certificate with size $\sum_a(d_a+1)$. Classical Helly,
  union-of-convex-set, fractional-Helly, and tolerance results supply the
  underlying intersection principles \cite{danzer1963helly,amenta1996short,
  eckhoff2009morris,montejano2011tolerance,kim2023leray}.

  \item We separate this audit from identification. A boundary proposition
  records that shared exogeneity alone admits every coupling. Once an
  external relation is declared, its relation-supported couplings give the
  exact identified set in a stated support-only two-regime feature model.
  No coupling inside that set is silently selected.

  \item We propagate simultaneous marginal confidence regions through the
  coupling program to obtain finite-sample outer confidence bounds for the
  complete oracle interval. Controlled and finite paired studies illustrate
  the local-to-global separation, while synthetic programs check the predicted
  interval and computational behavior. We report the empirical information
  boundary explicitly. The accompanying public repository contains the
  implementation, tests, configurations, and retained outputs; external
  datasets must be obtained from their original sources.
\end{enumerate}

The main line is
\[
\begin{aligned}
\text{regional costs}&\ \longrightarrow\
\text{common-witness certificate}\ \longrightarrow\
\text{declared feature relation}\\
&\ \longrightarrow\
\text{sharp coupling bounds}\ \longrightarrow\
\text{outer confidence bounds}.
\end{aligned}
\]
Technical extensions, full protocols, and secondary analyses are indexed in
the Supplementary Material.

\section{Related work and positioning}
\label{sec:related}

Pan and Bareinboim establish the nonidentification boundary for
counterfactual image editing and propose feature-level causal bounds
\cite{pan2024counterfactual}; later work develops a disentangled causal latent
editor \cite{pan2025bdcls}. Our object is different: an externally anchored
same-witness audit followed by a support-only partial-identification analysis.
It neither replaces their construction nor weakens their impossibility
theorem.

Partial identification of joint potential-outcome functionals has a long
history \cite{manski1990bounds,balke1997bounds,tian2000probabilities,
fan2010sharp,fan2017partial,zhang2022partial}. Coupling and transport methods
make the remaining cross-world ambiguity explicit
\cite{strassen1965existence,rachev1998mass,delara2024transport}. The linear
program and its transport dual are established tools. The contribution here
is their placement after an auditable, externally sourced relation, together
with a precise support-only sharpness statement.

Helly's theorem and fractional Helly theory control intersections of convex
families \cite{danzer1963helly,katchalski1979problem,kalai1984intersection}.
Amenta and Eckhoff--Nischke treat unions of convex components and the
generalized pigeonhole mechanism \cite{amenta1996short,eckhoff2009morris};
tolerance variants are also established
\cite{montejano2011tolerance,kim2023leray}. Our certificate theorems are
exact labeled specializations for the witness structures arising in the
audit, with sharpness and blocker-based repair. We do not claim that
finite-feature bounding, transport duality, or Helly theory is new.

Finally, \Cref{thm:outer-coverage} is confidence-set propagation. Its
importance here is the estimand: it covers the entire identified interval.
Clopper--Pearson bands are conservative but nonasymptotic
\cite{clopper1934use}; more efficient simultaneous multinomial regions can be
substituted if their joint coverage is proved
\cite{goodman1965simultaneous,sison1995simultaneous,may2000constructing}.

\section{Setting and the common-witness audit}
\label{sec:setting}

\subsection{Feature target and identification boundary}

Let $X\in\{0,1\}$ denote an intervention and $I_x\in\I$ the corresponding
potential image, where $\I$ is a standard Borel space. We prespecify a
measurable finite feature
\[
  \Psi:\I\to\Y,\qquad \Y=\{1,\ldots,K\},\qquad Y_x=\Psi(I_x),
\]
and assume that the single-world laws $\mu_x(y)=\Pp(Y_x=y)$ are identified
from randomized intervention data or another valid causal argument. Image
samples alone do not generally provide this identification.

Write $\pi_{yy'}=\Pp(Y_0=y,Y_1=y')$. Our targets are bounded linear
functionals
\begin{equation}
  H_h(\pi)=\sum_{y,y'}h(y,y')\pi_{yy'}.
  \label{eq:linear-target}
\end{equation}
Transition probabilities, average feature changes, and numerators of
conditional feature queries have this form. A conditional query with a
positive identified denominator $\mu_0(B)$ is obtained by dividing the
corresponding numerator by that fixed quantity.

\begin{proposition}[Shared-exogeneity saturation]
\label{prop:saturation}
Let $\mu_0,\mu_1$ be probability laws on a common standard Borel space.
Every coupling $\pi\in\Pi(\mu_0,\mu_1)$ is the joint potential-outcome law
of a two-regime structural model with one exogenous variable shared across
interventions.
\end{proposition}

\begin{proof}
On the probability space with law $\pi$, let $U=(U_0,U_1)$ be the coordinate
map and set $f(x,u_0,u_1)=u_x$. Then $Y=f(X,U)$ satisfies $Y_x=U_x$ and
$\mathcal L(Y_0,Y_1)=\pi$. Both worlds use the same realization of $U$.
\end{proof}

\Cref{prop:saturation} is a boundary lemma, not a novelty claim. It shows
that using the same witness twice cannot by itself overcome the causal
hierarchy. Information enters only through restrictions on admissible
exogenous states, response functions, or cross-world pairs.

\subsection{Regional losses, grade, and nerve}

Let $\W$ be a declared witness space and let
$c_r(w;i,j)\geq 0$, $r\in[m]=\{1,\ldots,m\}$, measure the violation of role
$r$ by witness $w$ for the factual--candidate pair $(i,j)$. When $(i,j)$ is
fixed or clear from context, we suppress it from the notation. Thus, for
example, $c_r(w)$ denotes $c_r(w;i,j)$; the same convention applies to all
witness sets, grades, and tolerant quantities defined below.

The ASCM response type $U$ and the auditing witness $w$ are distinct objects
unless an external scientific argument identifies them. At tolerance
$\varepsilon$, put
\begin{equation}
	A_r^\varepsilon(i,j)
	=\{w\in\W:c_r(w;i,j)\leq\varepsilon\}.
	\label{eq:witness-set}
\end{equation}
Compactness and lower semicontinuity are imposed whenever an attained
minimizer is used; the measurable version is given in Supplement S1.
\begin{definition}[Common-witness grade and nerve]
\label{def:grade}
For nonempty $S\subseteq[m]$, define
\begin{equation}
  F(S;i,j)=\inf_{w\in\W}\max_{r\in S}c_r(w;i,j),
  \qquad
  \rho(i,j)=F([m];i,j).
  \label{eq:grade}
\end{equation}
The witness nerve at scale $\varepsilon$ is
\begin{equation}
  K_\varepsilon(i,j)
  =\{\varnothing\}\cup
  \left\{\varnothing\ne S\subseteq[m]:
  \bigcap_{r\in S}A_r^\varepsilon(i,j)\ne\varnothing\right\}.
  \label{eq:nerve}
\end{equation}
\end{definition}

If $T\subseteq S$, then $F(T;i,j)\leq F(S;i,j)$. Consequently
$K_\varepsilon$ is a simplicial complex, maximal faces are maximal jointly
explainable role sets, and minimal nonfaces are minimal incompatibility
explanations. Under attainment, $(i,j)$ has one global witness exactly when
$\rho(i,j)\leq\varepsilon$.

To permit a prespecified number of regional exceptions, define
\begin{equation}
  F^{[s]}(S;i,j)=
  \inf_{w\in\W}
  \min_{\substack{D\subseteq S\\|D|\leq s}}
  \max_{r\in S\setminus D}c_r(w;i,j),
  \qquad \max\varnothing:=0,
  \label{eq:tolerant-grade}
\end{equation}
and $\rho_s=F^{[s]}([m])$. The integer $s$ is a within-pair role budget;
it is not a probability mass of inadmissible feature pairs.

\subsection{From image witnesses to a feature relation}

The audit induces the existential feature projection
\begin{equation}
  \R_\varepsilon=
  \left\{(y,y'):\begin{array}{l}
  \text{there is a lift $(i,j)$ with $\Psi(i)=y$, $\Psi(j)=y'$,}\\[-1mm]
  \text{and $\rho(i,j)\leq\varepsilon$}
  \end{array}\right\}.
  \label{eq:feature-relation}
\end{equation}
The causal model must separately assume
$\Pp\{(Y_0,Y_1)\in\R_\varepsilon\}=1$. The relation is therefore an
externally declared Layer-3 restriction, not a consequence of its definition.
An existential feature projection can also be an outer relaxation of the
image-level problem; it is lossless only under a feature-saturation or
conditional-fiber condition (Supplement S3). Without an external anchor,
the honest default is $\R_\varepsilon=\Y^2$.

\section{Combinatorial certificates and repair}
\label{sec:certificates}

We first give the finite-atlas result used directly by the experiments, then
place its convex analogue in the classical Helly context.

\subsection{Finite atlases}

\begin{theorem}[Exact finite-atlas tolerant certificate]
\label{thm:finite-atlas}
Suppose $\W$ is finite with $q=|\W|$ and $0\leq s<m$. Then
\begin{equation}
  \rho_s(i,j)=
  \max_{\substack{\varnothing\ne S\subseteq[m]\\
        |S|\leq\min\{q(s+1),m\}}}F^{[s]}(S;i,j).
  \label{eq:finite-atlas-certificate}
\end{equation}
Thus tolerant incompatibility has a certificate using at most $q(s+1)$
roles. The bound is best possible.
\end{theorem}

\begin{proof}
For fixed $w$, arrange its $m$ costs in nonincreasing order,
$c_{[1]}(w)\geq\cdots\geq c_{[m]}(w)$. Deleting at most $s$ roles leaves
the smallest possible maximum $c_{[s+1]}(w)$, hence
\[
  \rho_s=\min_{w\in\W}c_{[s+1]}(w).
\]
For every $w$, choose $T_w$ containing
$s+1$ roles with its largest costs and put $S_\star=\bigcup_wT_w$. Then
$|S_\star|\leq q(s+1)$. The $(s+1)$st largest cost of each $w$ on
$S_\star$ equals its $(s+1)$st largest cost on $[m]$, so
$F^{[s]}(S_\star)=\rho_s$. Monotonicity gives the reverse inequality for
every subfamily.

For sharpness, take $m=q(s+1)$ and partition the roles into disjoint blocks
$B_w$ of size $s+1$. Fix $0\leq b<d$ and set $c_r(w)=d$ for $r\in B_w$
and $c_r(w)=b$ otherwise. The complete family has tolerant grade $d$.
Every proper role set omits a role from some $B_w$; for that witness, at most
$s$ retained roles have cost $d$, so its tolerant grade is at most $b$.
Thus all $q(s+1)$ roles can be necessary.
\end{proof}

The proof is an exact consequence of the labeled finite-atlas structure. We
do not claim that tolerant Helly theory is new; generic tolerance transfers
and modern tolerance complexes are studied in
\cite{montejano2011tolerance,kim2023leray}. The audit-specific value of
\Cref{thm:finite-atlas} is its sharp certificate and the explicit repair
formula below.

\begin{proposition}[Blocker duality and exact repair]
\label{prop:blocker}
For fixed $\varepsilon$, let
\[
  B_w^\varepsilon=\{r:c_r(w;i,j)>\varepsilon\},
  \qquad G_w^\varepsilon=[m]\setminus B_w^\varepsilon.
\]
Then
\begin{equation}
  F^{[s]}(S;i,j)\leq\varepsilon
  \quad\Longleftrightarrow\quad
  \exists w\in\W:\ |S\cap B_w^\varepsilon|\leq s.
  \label{eq:blocker}
\end{equation}
Minimal $s$-tolerant incompatibility explanations are the inclusion-minimal
$(s+1)$-fold transversals of the bad-role hypergraph. Each has at most
$q(s+1)$ roles. For $s=0$,
\begin{equation}
  K_\varepsilon=\bigcup_{w\in\W}2^{G_w^\varepsilon},
  \qquad
  s_\varepsilon^\star
  :=\min\{s:\rho_s\leq\varepsilon\}
  =\min_{w\in\W}|B_w^\varepsilon|.
  \label{eq:repair}
\end{equation}
\end{proposition}

\begin{proof}
For fixed $w$, all retained costs are at most $\varepsilon$ after at most
$s$ deletions exactly when at most $s$ members of $S$ belong to
$B_w^\varepsilon$. This proves \eqref{eq:blocker}; negating it yields the
multiple-transversal description. Choosing $s+1$ bad roles for every
witness gives the size bound. Finally, $|B_w^\varepsilon|$ is exactly the
number of deletions required for witness $w$, and minimizing over $w$ proves
\eqref{eq:repair}.
\end{proof}

Rowwise order selection computes $\rho_s$, one valid certificate, and the
exact repair number in $O(qm)$ selection time (or $O(qm\log m)$ by sorting).
Enumeration of all minimal transversals can still be exponential
\cite{amaldi2003maximum}.

\subsection{Convex and action-stratified atlases}

For a compact convex witness set of affine dimension $d$, continuous
quasiconvex role losses have convex closed sublevel sets. The classical
Helly theorem then gives
\begin{equation}
  \rho(i,j)=
  \max_{\substack{\varnothing\ne S\subseteq[m]\\|S|\leq d+1}}F(S;i,j).
  \label{eq:helly-grade}
\end{equation}
Indeed, at the maximum local grade every $d+1$ sublevel sets intersect, so
the complete family intersects. Equation \eqref{eq:helly-grade} is a direct
application of Helly's theorem, not a new intersection theorem.

Many audits have a discrete requested action and continuous nuisance
parameters. The resulting witness space is a labeled disjoint union rather
than one convex set.

\begin{theorem}[Heterogeneous action-stratified certificate]
\label{thm:action-stratified}
Suppose
$\W=\bigsqcup_{a=1}^q\W_a$, where $\W_a$ is nonempty, compact, convex,
and has affine dimension $d_a$. Assume every role loss is continuous and
quasiconvex on each stratum. Define
\[
  F_a(S)=\min_{w\in\W_a}\max_{r\in S}c_r(w),
  \qquad F(S)=\min_aF_a(S),
\]
and $H=\min\{m,\sum_{a=1}^q(d_a+1)\}$. Then
\begin{equation}
  F([m])=
  \max_{\substack{\varnothing\ne S\subseteq[m]\\|S|\leq H}}F(S).
  \label{eq:action-certificate}
\end{equation}
The bound is sharp for every dimension list when
$m=\sum_a(d_a+1)$.
\end{theorem}

\begin{proof}
Let $\rho_a=F_a([m])$. Applying \eqref{eq:helly-grade} within stratum $a$
gives $S_a\subseteq[m]$, $|S_a|\leq d_a+1$, with
$F_a(S_a)=\rho_a$. Put $S_\star=\bigcup_aS_a$. Monotonicity gives
\[
  \rho_a=F_a(S_a)\leq F_a(S_\star)\leq F_a([m])=\rho_a
\]
for every $a$. Hence $F(S_\star)=\min_a\rho_a=F([m])$ and
$|S_\star|\leq H$.

For sharpness, create one block of $d_a+1$ roles $(a,j)$ per stratum and
take $\W_a=\Delta^{d_a}$. For $x\in\Delta^{d_b}$ set
\[
  c_{(a,j)}(b,x)=
  \begin{cases}(d_a+1)x_j,&a=b,\\0,&a\ne b.\end{cases}
\]
The complete block has grade one in its stratum, so the full-family grade is
one. If $(a,j)$ is removed, choose stratum $a$ and vertex $e_j$; every
remaining cost is zero. Thus every proper subfamily has grade zero.
\end{proof}

Classical results for unions of convex sets already imply closely related
$q(d+1)$ Helly numbers \cite{amenta1996short,eckhoff2009morris}. The content
of \Cref{thm:action-stratified} is the audit-specialized exact grade equality,
heterogeneous dimension sum, and sharp labeled construction, not a claim of a
fundamentally new general Helly theorem.

The witness nerve also quantifies approximate coherence. Let $h=d+1$ and
let $N_h(\varepsilon)$ count its feasible $h$-faces. If
$D_\varepsilon$ is the largest number of roles explained by one witness and
$t_\varepsilon=m-D_\varepsilon$, Kalai's exact fractional-Helly bound gives
\begin{equation}
  N_h(\varepsilon)
  \leq \binom{m}{h}-\binom{d+t_\varepsilon}{h}.
  \label{eq:fractional-repair}
\end{equation}
In particular, if at least an $\alpha$ fraction of the $h$-subsets are
feasible, one witness explains at least
\begin{equation}
  \left\lceil\bigl[1-(1-\alpha)^{1/(d+1)}\bigr]m\right\rceil
  \label{eq:fractional-core}
\end{equation}
roles \cite{katchalski1979problem,kalai1984intersection,eckhoff1985upper}.
The complete proof and exact integer inversion are in Supplement S4. Dense
local compatibility yields a large coherent core, not global compatibility.

If estimated losses obey
$\max_r\sup_w|\widehat c_r(w;i,j)-c_r(w;i,j)|\leq\delta$, then every
tolerant grade changes by at most $\delta$ and
\begin{equation}
  K^{[s]}_{\varepsilon-\delta}
  \subseteq\widehat K^{[s]}_\varepsilon
  \subseteq K^{[s]}_{\varepsilon+\delta}.
  \label{eq:nerve-stability}
\end{equation}
This is a deterministic robustness statement; statistical use requires an
independently justified uniform error bound. Finite-net outer approximation
and exact nerve recovery away from critical grades are in Supplement S4.

\section{Sharp feature bounds}
\label{sec:bounds}

For a declared relation $\R\subseteq\Y^2$, let
\begin{equation}
  \Pi_{\R}(\mu_0,\mu_1)=
  \left\{\pi\in\mathbb R_+^{K\times K}:
  \sum_{y'}\pi_{yy'}=\mu_0(y),\quad
  \sum_y\pi_{yy'}=\mu_1(y'),\quad
  \pi_{yy'}=0\ \text{outside }\R\right\}.
  \label{eq:relation-couplings}
\end{equation}
Equivalently, $\pi\mathbf 1=\mu_0$ fixes row sums and
$\pi^\top\mathbf 1=\mu_1$ fixes column sums. The transpose is needed
because left multiplication by $\pi^\top$ sums the original matrix down its
rows, producing one total for each column.

The \emph{support-only feature model} contains every two-regime feature SCM
with these marginals and $\Pp\{(Y_0,Y_1)\in\R\}=1$, with no other response
function, latent-DAG, or full-image restriction.

\begin{theorem}[Sharp support-only identified interval]
\label{thm:sharp}
If $\Pi_{\R}(\mu_0,\mu_1)$ is nonempty, then the sharp identified set for
\eqref{eq:linear-target} in the support-only feature model is
\begin{equation}
  \left[
  \min_{\pi\in\Pi_{\R}(\mu_0,\mu_1)}H_h(\pi),
  \max_{\pi\in\Pi_{\R}(\mu_0,\mu_1)}H_h(\pi)
  \right].
  \label{eq:sharp-interval}
\end{equation}
Both endpoints and every intermediate value are attainable.
\end{theorem}

\begin{proof}
The feasible set is a nonempty compact convex polytope and $H_h$ is linear,
so its image is a closed interval with attained endpoints. Every model in
the declared class induces a feasible coupling. Conversely,
\Cref{prop:saturation} realizes every feasible coupling as a shared-exogenous
feature SCM, and mixtures realize all intermediate values.
\end{proof}

Sharpness is relative to the displayed model. If a fixed full-image law,
latent DAG, or nonsaturated image relation imposes further restrictions, the
feature program can be only outer. This qualification prevents an audit
relation from being mistaken for identification of the Pan--Bareinboim pixel
counterfactual.

If the external science supports only a violation-mass budget $\tau$, replace
hard support by
\begin{equation}
  \Pi_{\R}^{(\tau)}(\mu_0,\mu_1)
  =\{\pi\in\Pi(\mu_0,\mu_1):\pi(\R^c)\leq\tau\}.
  \label{eq:contamination-class}
\end{equation}
The same compactness argument makes its optimized interval sharp in the
corresponding budget model and nested in $\tau$. The smallest feasible
budget is the Hall deficiency
\[
  \delta_{\R}(\mu_0,\mu_1)
  =\max_{A\subseteq\Y}\{\mu_0(A)-\mu_1(\R(A))\},
\]
by max-flow/min-cut and the finite Hall--Strassen criterion
\cite{strassen1965existence}. These established transport results, their
duals, and the Polish-space extension are collected in Supplements S2--S3.
The parameter $\tau$ is sensitivity input, not a probability learned from
unpaired images.

For sparse $\R$, the two endpoint programs use one variable per allowed edge
and $2|\R|$ nonzeros in the marginal equality matrix. The relation is first
checked for feasibility; infeasibility is reported rather than hidden by
renormalization.

\section{Finite-sample outer inference}
\label{sec:inference}

Let $\mathcal C_\alpha$ be any random simultaneous confidence region for the
two regime marginals satisfying
\begin{equation}
  \Pp\{(\mu_0,\mu_1)\in\mathcal C_\alpha\}\geq1-\alpha.
  \label{eq:marginal-confidence-region}
\end{equation}
For a fixed known relation define the union of compatible couplings
\begin{equation}
  \widehat{\mathcal P}_\alpha(\R)=
  \bigcup_{(\nu_0,\nu_1)\in\mathcal C_\alpha}
  \Pi_{\R}(\nu_0,\nu_1).
  \label{eq:outer-couplings}
\end{equation}

\begin{theorem}[Outer coverage of the complete oracle interval]
\label{thm:outer-coverage}
Let $[L^\star,U^\star]$ be the sharp interval over
$\Pi_{\R}(\mu_0,\mu_1)$. If the random endpoints are measurable, then
\begin{equation}
  \Pp\left\{[L^\star,U^\star]\subseteq
  \left[
  \inf_{\pi\in\widehat{\mathcal P}_\alpha(\R)}H_h(\pi),
  \sup_{\pi\in\widehat{\mathcal P}_\alpha(\R)}H_h(\pi)
  \right]\right\}\geq1-\alpha.
  \label{eq:outer-coverage}
\end{equation}
If the random feasible set is empty, reporting the vacuous payoff range
preserves the guarantee.
\end{theorem}

\begin{proof}
On the event in \eqref{eq:marginal-confidence-region}, every oracle-feasible
coupling belongs to \eqref{eq:outer-couplings}. Minimizing over the larger
set cannot increase the lower endpoint, and maximizing cannot decrease the
upper endpoint.
\end{proof}

A distribution-free concrete choice uses independent within-regime samples.
If $N_x(y)$ is the count in cell $y$ among $n_x$ observations, construct a
two-sided Clopper--Pearson interval for each of the $M=2K$ marginal cells at
cellwise noncoverage $\alpha/M$. In beta-quantile notation its endpoints are
\begin{align}
 \ell_x(y)&=
 \begin{cases}0,&N_x(y)=0,\\
 \operatorname{Beta}^{-1}\!\left(\frac{\alpha}{2M};
 N_x(y),n_x-N_x(y)+1\right),&N_x(y)>0,
 \end{cases}\label{eq:cp-lower}\\
 u_x(y)&=
 \begin{cases}1,&N_x(y)=n_x,\\
 \operatorname{Beta}^{-1}\!\left(1-\frac{\alpha}{2M};
 N_x(y)+1,n_x-N_x(y)\right),&N_x(y)<n_x.
 \end{cases}\label{eq:cp-upper}
\end{align}
Each count is marginally binomial, so Bonferroni gives simultaneous coverage
despite dependence among cells within a multinomial sample
\cite{clopper1934use}. Here, exact means finite-sample coverage of at
least the nominal level, not equality or shortest possible width. Hoeffding
bands provide a simpler alternative. The inference target is the complete
oracle interval, not one selected coupling.

If a random outer relation satisfies
$\Pp(\R_\star\subseteq\widehat\R^+)\geq1-\beta$, the same containment
argument and a union bound give coverage at least $1-\alpha-\beta$ when the
program uses $\widehat\R^+$. Sample splitting alone does not establish this
outer-relation property. Compatibility tests and independent finite-panel
variants, with their required sampling assumptions, are in Supplement S4.

\section{Audit and optimization pipeline}
\label{sec:algorithm}

The operational procedure keeps its information sources separate:

\begin{enumerate}[leftmargin=1.6em]
  \item Prespecify $\Psi$, protected roles, allowed descendants, the witness
  atlas, tolerance, and any role or relation-violation budget.
  \item For a candidate pair compute $\rho$ or $\rho_s$ and return a
  short blocking-role certificate and exact repair count when the audit
  fails.
  \item Project the externally justified audit to a feature relation and
  state explicitly whether that projection is exact or outer.
  \item Estimate the two single-world marginals and solve the lower and upper
  support or budget transport programs, using simultaneous marginal bands
  when finite-sample coverage is required.
  \item Return the identified interval, outer confidence interval, and
  incompatibility explanation. Selecting a point inside the interval
  requires a separate declared decision rule.
\end{enumerate}

For nonconvex neural witness spaces not represented by a verified finite
atlas, a verified global optimizer would be needed; the certificates above do
not validate an arbitrary local neural search.

\section{Numerical studies}
\label{sec:numerical-studies}

The numerical studies address three questions that correspond directly to the
theory: whether regional plausibility can coexist with global
incompatibility, whether a declared relation can sharpen feature-level bounds
without concealing infeasibility or misspecification, and whether the
resulting optimization and certificate computations remain tractable in
structured large instances.  The studies are not presented as evidence that
an unrestricted pixel-level counterfactual is identified.

\subsubsection*{Status of the numerical evidence}
The accompanying public repository contains the implementation, tests,
configurations, retained outputs, and integrity manifests supporting the
numerical studies. External datasets are not redistributed and must be
obtained from their original sources. The repository documents the scope of
the retained evidence and the limitations of exact historical and
cross-platform replay.

\subsection{Common-witness audits}
\label{sec:main-common-witness-audits}

\paragraph{Controlled rotation audit.}
We first use the official MNIST training and test partitions
\cite{lecun1998gradient} in a controlled renderer experiment.  For a source
image \(S\), one angle
\[
  A\in\{-20,-10,0,10,20\}\text{ degrees}
\]
is applied to the whole image, followed by clipped independent Gaussian
measurement noise with standard deviation \(0.02\).  The witness atlas is the
same declared set of five renderer responses.  For rectangular region \(P_r\),
the discrepancy of angle \(a\) is
\[
  c_r(a)=
  \left[|P_r|^{-1}\sum_{p\in P_r}
  \{I_1(p)-\operatorname{Rotate}(I_0,a)(p)\}^2\right]^{1/2}.
\]
The local and common-witness scores are
\[
  L=\max_r\min_a c_r(a),
  \qquad
  \rho=\min_a\max_r c_r(a).
\]
Thus \(L\) permits a different angle in each region, whereas \(\rho\) requires
one angle to explain every region.  A split of \(3{,}000\) training images
sets the \(0.975\) order-statistic threshold and a disjoint \(3{,}000\)-image
split assesses relation violations.  The archived protocol evaluates the
official \(10{,}000\) test images at three seeds and four fixed partitions
(\(4,9,16,\) and \(64\) regions).  These are twelve configurations, not
\(120{,}000\) independent test units: the same \(10{,}000\) images recur.
The principal negative control selects angles separately by region, so it is
constructed to satisfy the local quantifier while violating the common-angle
quantifier.  Because every split uses the same specified renderer, this is
held-out validation within a controlled mechanism, not independent scientific
validation of a causal relation.

\paragraph{Separately fitted editor and witness.}
The second study uses paired responses supplied by Morpho-MNIST
\cite{castro2019morphomnist} and smallNORB \cite{lecun2004learning}.
Morpho-MNIST provides index-matched plain, thin, and thick digits; these are
benchmark-generated transformations rather than physical interventions.
smallNORB provides physical toy objects photographed under factorially varied
pose and illumination.  Lighting \(0\) is paired with lightings \(1,3,\) and
\(5\), holding object, pose, and camera fixed.  These are matched photographs,
not observations of the same stochastic unit in two counterfactual worlds.

An action-conditional latent-residual editor is fitted on units disjoint from
a separate PCA--ridge witness.  The editor is an experimental vehicle rather
than a claimed architectural contribution.  The witness divides each image
into a fixed \(4\times4\) grid and normalizes each regional discrepancy by an
action-specific held-out threshold \(q_a\).  For normalized costs
\(\widetilde c_r(a)=c_r(a)/q_a\), the relevant scores are
\[
  L=\max_r\min_a\widetilde c_r(a),\qquad
  \rho_{\mathrm{free}}=\min_a\max_r\widetilde c_r(a),\qquad
  \rho_{\mathrm{req}}=\max_r
  \widetilde c_r(a_{\mathrm{req}}).
\]
A strictly positive \(q_a\) is required; the protocol uses a positive
calibration quantile, and any zero quantile would require a positive floor
fixed before evaluation.
A score at most one is accepted.  The free score asks whether some declared
action explains the complete image; only the requested-action score tests
compliance with the requested edit.  The negative control alternates paired
actions across the sixteen regions.  It can therefore be locally plausible
even though no single action explains the image.

The relation panel contains \(4{,}000\) Morpho-MNIST base digits and ten
smallNORB physical objects.  The editor evaluation uses a disjoint \(6{,}000\)
Morpho-MNIST base digits and fifteen smallNORB physical objects, with three
prespecified editor seeds.  Repeated actions and views are dependent
measurements of the same base digit or object and are not counted as new
independent units.  In particular, the smallNORB threshold was calibrated
from only five physical objects and is an empirical rule, not a
distribution-free conformal guarantee.

\begin{table}[t]
\centering
\small
\begin{tabular}{@{}lcrrrr@{}}
\toprule
Study & Evaluation unit & Coherent & Patch local & Patch global & AUROC \\
\midrule
MNIST rotations & \(10{,}000\)/setting
  & \(0.9709\)--\(0.9777\) & \(0.9702\)--\(0.9804\)
  & \(0\)--\(0.0018\) & \(0.9988\)--\(1.0000\) \\
Morpho-MNIST & \(4{,}000\) digits
  & \(0.977750\) & \(0.976750\) & \(0.113750\) & \(0.987213\) \\
smallNORB & \(10\) objects
  & \(0.959259\) & \(0.972840\) & \(0\) & \(1.000000\) \\
\bottomrule
\end{tabular}
\caption{Archived aggregate common-witness results.  Coherent is the
free-angle global acceptance for controlled MNIST and requested-action
acceptance for the paired studies; patch global uses the free common-action
score.  MNIST entries are ranges over twelve
seed--partition configurations that reuse the same test images.  The
smallNORB AUROC of \(1\) is complete separation on a fixed ten-object panel,
not a population-perfect guarantee.}
\label{tab:main-common-witness-results}
\end{table}

Table~\ref{tab:main-common-witness-results} shows the intended
local-to-global separation in all three studies.  In the controlled renderer,
patchworks retain approximately the same local acceptance as coherent pairs
but almost never admit one global angle.  The learned-witness study shows the
same qualitative separation, although the Morpho-MNIST free-action global
acceptance of \(0.113750\) is materially above zero.  On the disjoint editor
sets, the conditional editor's archived mean whole-image SSIM is \(0.903131\)
versus \(0.870558\) for conditional ridge on Morpho-MNIST, and \(0.986064\)
versus \(0.985821\) on smallNORB.  We report the corresponding improvements
only to appropriate precision, \(0.0326\) and \(0.00024\): the latter is
practically very small.

The smallNORB aggregate also conceals important object-level uncertainty.
Requested-action acceptance ranges from \(0.779835\) to \(1\) across ten
objects.  Zero global patchwork acceptances among ten objects has one-sided
\(95\%\) Clopper--Pearson upper limit \(0.2589\).  Consequently, the reported
AUROC \(1\) supports finite-panel separation of the constructed control but
does not establish population-perfect detection or a population patchwork
acceptance below \(0.20\).

\subsection{Sharp bounds and finite-sample outer inference}
\label{sec:main-bound-experiments}

A three-state synthetic feature provides a direct numerical check of the
support-constrained coupling program.  With only the two regime marginals,
the sharp target interval is \([0.35,0.80]\).  Imposing the declared one-step
relation narrows it to \([0.55,0.70]\).  A misspecification control places
\(0.20\) of the true coupling mass outside that relation, and the constrained
interval then fails to cover the true target.  The negative control is
essential: narrowing is created by the added relation, not by the observed
marginals alone.

For finite-sample inference, independent samples from each regime are drawn
at
\[
  n\in\{100,300,1000,3000,10000,30000\},
\]
with \(300\) repetitions per sample size.  The archived Hoeffding outer
intervals cover the complete
oracle identified interval in all \(1{,}800\) repetitions, with mean width
decreasing from \(0.9237\) to \(0.2075\).  A later analysis applies the exact
Bonferroni--Clopper--Pearson construction in
\eqref{eq:cp-lower}--\eqref{eq:cp-upper} to the same archived cell counts. Its
reported mean widths decrease from \(0.7413\) to \(0.1861\), reductions of
\(10.3\%\)--\(27.8\%\) relative to Hoeffding, and all \(1{,}800\) intervals
again cover.  These successes are an implementation check under the stated
simulation law, not evidence of exact nominal calibration; the coverage
guarantee follows from the theorem.  The exact-band comparison is
post-confirmatory and descriptive because it was specified after the
Hoeffding outcomes had been inspected.

\begin{figure}[t]
  \centering
  \includegraphics[width=.98\linewidth]
  {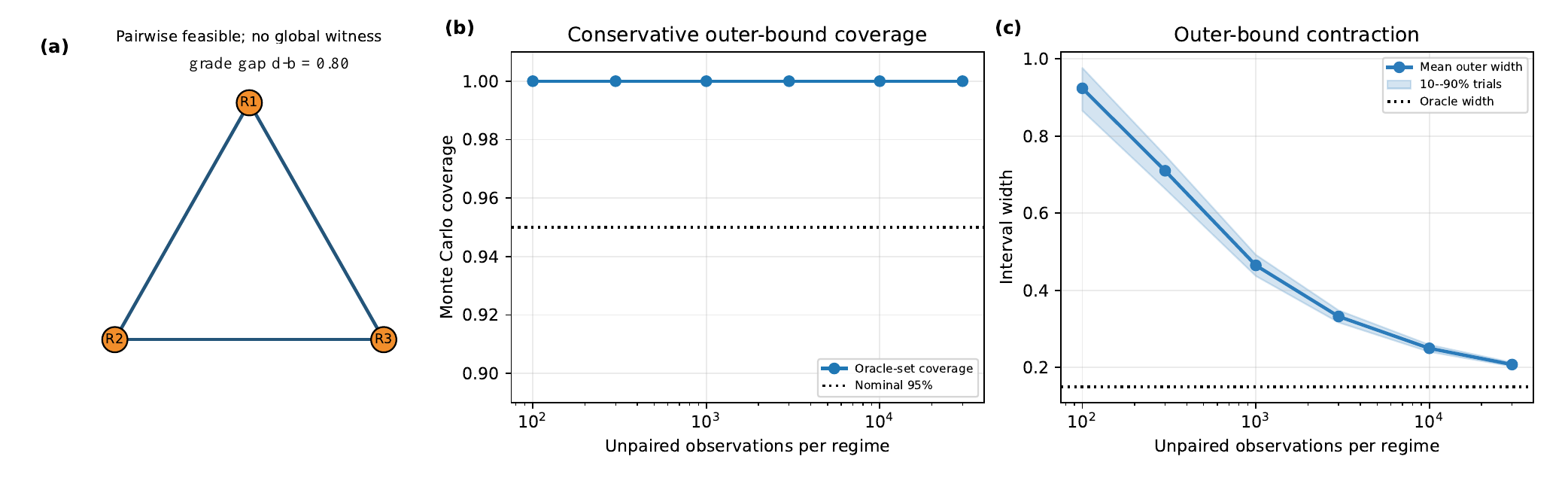}
  \caption{Archived synthetic summaries.  Left: all pairs can be compatible
  while the complete witness intersection is empty.  Middle and right:
  observed coverage of the complete oracle interval and contraction of the
  Hoeffding outer interval as the regime sample sizes increase.}
  \label{fig:main-synthetic-summary}
\end{figure}

The controlled audit-to-bound panel also records a necessary negative result.
It fixes ten images per digit, for \(N=100\) units, and the raw witness
relation contains \(0.0097\) of the \(N^2\) candidate pairs.  The hard-support
transport program is infeasible and is reported as such; no renormalization
or silent relaxation is used.  With the predeclared empirical violation
budget \(\tau=0.08\), the archived interval is \([0.1517,0.2600]\), compared
with \([0.0600,0.8200]\) without the image relation, and contains the hidden
paired target \(0.2100\).  However, balancing the panel by digit violates the
i.i.d.\ premise of the intended pooled binomial guarantee.  A post-hoc
stratified sensitivity calculation uses \(\tau=0.12\) and gives
\([0.1183,0.2975]\), also containing \(0.2100\).  Neither relaxed result is a
confirmatory \(95\%\) confidence statement; a new independent panel with the
stratified procedure fixed in advance would be required.

\subsection{Structured computational checks}
\label{sec:main-scaling}

The archived computations use sparse edge variables for banded transport and
short dual certificates for repeated-simplex minimax instances.  These tests
check the implementations against known optima and residual conditions; they
do not claim comparable scaling for dense arbitrary relations, face
enumeration, or nonconvex neural witness optimization.

\begin{table}[t]
\centering
\small
\begin{tabular}{@{}lrrr@{}}
\toprule
Structured problem & Largest instance & Representation & Time (s) \\
\midrule
Sparse coupling & \(10{,}000\) states & \(109{,}970\) edges & \(18.521\) \\
Affine minimax, \(d=4\) & \(100{,}000\) roles & \(5\)-role certificate & \(0.251\) \\
Affine minimax, \(d=8\) & \(100{,}000\) roles & \(9\)-role certificate & \(0.385\) \\
Affine minimax, \(d=16\) & \(100{,}000\) roles & \(17\)-role certificate & \(0.786\) \\
Helly-tight & \(1{,}024\) roles & dimension \(1{,}023\) & \(9.704\) \\
Finite-atlas sharpness & \(12\) roles & \(4{,}094\) proper faces & \(0.192\) \\
\bottomrule
\end{tabular}
\caption{Largest successful cases in the archived hardware-specific scaling
summary.  Relations and affine certificate families are supplied rather than
learned.  Times are descriptive and hardware-specific.}
\label{tab:main-scaling}
\end{table}

All \(35\) archived scaling cases are reported as successful in \(34.85\)
seconds total with peak resident memory \(0.483\) GiB.  At \(10{,}000\)
states, sparse transport replaces \(10^8\) dense state pairs by \(109{,}970\)
edge variables; the reported marginal residuals are below
\(1.4\times10^{-18}\).  Across the affine cases, the largest reported primal,
stationarity, or duality error is \(3.34\times10^{-16}\).  Constructed
leave-one-out atlases with \(m=4,\ldots,12\) accept every nonempty proper role
set and reject the full set, attaining the finite-atlas certificate bound.
These are numerical verification and structured-scaling results, not an
empirical claim about naturally occurring high-order image obstructions.

\subsection{Empirical scope}
\label{sec:main-empirical-scope}

The experiments validate the declared computations and local-to-global
failure mode on controlled or finite paired response families.  They do not
establish the scientific correctness of an arbitrary relation, identify an
unrestricted counterfactual image, or infer a joint multi-regime image law.
Morpho-MNIST is synthetic, and smallNORB has only ten physical objects in the
relation panel.  The reported large-state computations rely on supplied
sparse or repeated structure.

\section{Scope, information boundary, and limitations}
\label{sec:limitations}

We first clarify the scope of the results. Unrestricted pixel-level
counterfactuals are generally not identified from unpaired regime marginals.
Accordingly, the framework targets sharp bounds for prespecified
finite-dimensional image features under a declared support restriction. This
is the identified object of the analysis rather than an approximation to an
otherwise identified pixel-level counterfactual.

The witness atlas and the cross-world relation are additional scientific
inputs, not consequences of the observed regime marginals. The witness atlas
determines the coherence question being audited, whereas the relation
determines the admissible feature couplings used for partial identification.
If either is misspecified, the audit may reject coherent pairs or the
identified interval may exclude the true target. Sample splitting can assess
empirical performance but cannot by itself establish the causal validity of
these inputs. Likewise, violation budgets are sensitivity parameters unless
they are independently calibrated.

A further information boundary arises from feature projection. Projecting an
image-level relation onto a coarse feature space can discard image-level
restrictions. The feature-level and image-level analyses coincide only under
the feature-saturation or conditional-fiber conditions stated in the
Supplementary Material.

The current experiments provide controlled evaluations on MNIST,
Morpho-MNIST, and smallNORB, together with synthetic and computational
studies. They validate the predicted local--global separation, sharp-bound
calculations, finite-sample coverage, and structured computational scaling in
these settings. Evaluation on broader natural-image domains and empirical
validation of the multi-regime extension are left for future work.

\section{Conclusion}
\label{sec:conclusion}

Local regional plausibility need not imply one globally coherent
counterfactual explanation. The common-witness grade records this distinction
as a feasibility complex. Finite and action-stratified witness structures
yield short exact certificates and repair counts; an externally justified
feature relation then narrows, but does not select within, the coupling of
the regime marginals. Optimizing over all compatible couplings gives sharp
feature bounds in the stated support-only model, and simultaneous marginal
confidence regions give finite-sample outer coverage of the complete oracle
interval. The resulting pipeline is mathematically explicit about where
information enters and where ambiguity remains. Its validity in a new
application rests on prespecification, external validation of the witness
relation, and reproducible evidence at the correct independent unit.

\section*{Code and data availability}

The implementation code, experiment runners, tests, protocols, retained aggregate outputs, and reproduction instructions are available \href{https://github.com/joseffaghihi/Common-Witness-Certificatesand-Sharp-Feature-Bounds-for-Counterfactual-Image-Auditing}{here}. External datasets are not redistributed and must be obtained from the original sources cited in the Supplementary Material. Documented reproduction limitations are provided in the repository.

\bibliographystyle{siamplain}
\bibliography{references}

\end{document}